%% file: main.tex
\documentclass[runningheads]{llncs}

\usepackage[T1]{fontenc}
\usepackage{graphicx}
\input{notations}

\begin{document}
\title{
On the Generalization Bounds of Symbolic Regression with Genetic Programming
}
\titlerunning{
On the Generalization Bounds of Symbolic Regression with GP
}
%
\author{
Masahiro Nomura\inst{1}
\and
Ryoki Hamano\inst{2}
\and
Isao Ono\inst{1}
}
\authorrunning{Nomura et al.}
%
\institute{
Institute of Science Tokyo, Yokohama, Japan\\
\email{\{nomura,isao\}@comp.isct.ac.jp}
\and
CyberAgent, Shibuya, Japan\\
\email{hamano\_ryoki\_xa@cyberagent.co.jp}
}
\maketitle              

\input{manuscripts/00_abst}
\input{manuscripts/01_intro}

\input{manuscripts/02_setup}

\input{manuscripts/03_result}

\input{manuscripts/04_proof}

\input{manuscripts/05_discussion}
\input{manuscripts/06_conclusion}

\section*{Acknowledgement}
We are grateful to the anonymous reviewers for their helpful comments and suggestions.
We are also grateful to Peter Rockett for valuable feedback.

\bibliographystyle{splncs04}
\bibliography{ref}

\end{document}

%% file: notations.tex
\usepackage[utf8]{inputenc}
\usepackage[T1]{fontenc}
\usepackage{lmodern}
\usepackage{microtype}

\usepackage{amsmath, amssymb, mathtools, bm}
\usepackage{amscd}

\usepackage{bbm}
\usepackage{enumitem}
\usepackage{booktabs}
\usepackage{algorithm}
\usepackage{algorithmic}
\usepackage{hyperref}
\usepackage[nameinlink]{cleveref}
\usepackage{cite}

\hypersetup{
  colorlinks=true,
  linkcolor=blue,
  citecolor=blue,
  urlcolor=blue
}

\renewcommand{\leq}{\leqslant} 
\renewcommand{\ge}{\geqslant} 
\renewcommand{\le}{\leqslant} 

\usepackage{tikz}
\usepackage{forest}

\newcommand{\R}{\mathbb{R}}
\newcommand{\N}{\mathbb{N}}
\newcommand{\E}{\mathbb{E}}

\newcommand{\norm}[1]{\left\lVert #1 \right\rVert}
\newcommand{\abs}[1]{\left\lvert #1 \right\rvert}
\newcommand{\cF}{\mathcal{F}}
\newcommand{\cG}{\mathcal{G}}
\newcommand{\cT}{\mathcal{T}}
\newcommand{\cO}{\mathcal{O}}
\newcommand{\cX}{\mathcal{X}}
\newcommand{\cY}{\mathcal{Y}}
\newcommand{\cD}{\mathcal{D}}
\newcommand{\cL}{\mathcal{L}}
\newcommand{\cC}{\mathcal{C}}

\newcommand{\Rad}{\mathfrak{R}}
\newcommand{\what}{\widehat}
\newcommand{\eps}{\varepsilon}

%% file: manuscripts/00_abst.tex
\begin{abstract}
Symbolic regression (SR) with genetic programming (GP) aims to discover interpretable mathematical expressions directly from data.
Despite its strong empirical success, the theoretical understanding of why GP-based SR generalizes beyond the training data remains limited.
In this work, we provide a learning-theoretic analysis of SR models represented as expression trees.
We derive a generalization bound for GP-style SR under constraints on tree size, depth, and learnable constants.
Our result decomposes the generalization gap into two interpretable components:
a structure-selection term, reflecting the combinatorial complexity of choosing an expression-tree structure, and a constant-fitting term, capturing the complexity of optimizing numerical constants within a fixed structure.
This decomposition provides a theoretical perspective on several widely used practices in GP, including parsimony pressure, depth limits, numerically stable operators, and interval arithmetic.
In particular, our analysis shows how structural restrictions reduce hypothesis-class growth while stability mechanisms control the sensitivity of predictions to parameter perturbations.
By linking these practical design choices to explicit complexity terms in the generalization bound, our work offers a principled explanation for commonly observed empirical behaviors in GP-based SR and contributes towards a more rigorous understanding of its generalization properties.
\keywords{Symbolic regression \and Genetic programming \and Generalization bound.}
\end{abstract}

%% file: manuscripts/01_intro.tex
\section{Introduction}

Symbolic regression (SR)~\cite{kronberger2024symbolic} aims to discover analytical expressions that describe relationships between input variables and target outputs directly from data. 
Unlike conventional regression methods, which assume a predefined model class and then optimize parameters within that class, SR simultaneously searches for both the functional form of the predictor and its numerical coefficients.
Because the output of SR is an explicit mathematical expression, the resulting models are often interpretable and can sometimes reveal underlying mechanisms in the data.
For this reason, SR has long been regarded as a promising approach for interpretable modeling and scientific discovery, where compact analytical expressions can provide insight in addition to predictive accuracy~\cite{udrescu2020ai,brunton2016discovering,wang2019symbolic}.

Genetic programming (GP)~\cite{koza1992genetic,poli2008field} is one of the most widely used approaches to SR.
In GP-based SR, candidate models are typically represented as expression trees whose internal nodes correspond to operators and whose leaves correspond to variables or constants.
Evolutionary search is then used to explore the space of symbolic structures through operations such as mutation, crossover, and selection.
In addition, the numerical constants within a given structure are often refined using dedicated continuous optimization techniques, such as nonlinear least squares~\cite{kommenda2020parameter}.
This hybrid nature---combining discrete structure search with continuous parameter optimization---makes GP particularly attractive for SR, as it enables open-ended exploration of functional forms without requiring prior commitment to a specific model class.
However, the same flexibility that allows GP to discover concise symbolic models also raises a fundamental question: why should a model found by GP generalize beyond the training data?
In learning-theoretic terms, this question concerns how well a model that fits the training data will perform on unseen data.
This is typically formalized through the notion of a generalization gap, the difference between training error and expected prediction error, and the goal of deriving bounds that hold uniformly over a class of models.
In practice, GP-based SR is known to be sensitive to design choices such as tree-size limits, maximum depth, constant optimization strategies, and the use of numerically stable operators, many of which are motivated by the need to control overfitting.

These observations naturally lead to the question of how such design choices influence generalization from a theoretical perspective.
The GP literature has proposed various techniques intended to improve the generalization ability of SR models, including parsimony pressure, depth limits, and operator design aimed at improving numerical stability~\cite{poli2008parsimony,luke2002lexicographic,keijzer2003improving}.
These approaches are widely used and often effective in practice.
Nevertheless, their effects are typically justified empirically rather than through a unified theoretical framework.
One reason is that SR induces a hypothesis class with a hybrid structure, combining discrete structural choices with continuous parameter optimization, which makes standard learning-theoretic tools difficult to apply directly.
As a result, although many practical heuristics for controlling overfitting in GP-based SR have been developed, there remains limited learning-theoretic understanding of how these two aspects jointly influence generalization.
Understanding this interplay is essential for developing principled methods for controlling model complexity in SR.

In this paper, we present a learning-theoretic analysis of SR models represented by expression trees.
To the best of our knowledge, this is among the first works to derive an explicit generalization bound for GP-style SR models.
Our main result provides a bound that decomposes the generalization gap into two interpretable components corresponding to structure selection and constant fitting.
Specifically, the bound consists of a term that scales with the sensitivity of the model to its learnable constants and a term that depends on the logarithm of the number of admissible tree structures, reflecting the combinatorial complexity of structure selection.
This decomposition offers a theoretical lens for understanding how common mechanisms used in GP-based SR influence generalization.
In particular, structural control techniques such as parsimony pressure or depth limits restrict the space of admissible tree structures, while numerical-stability mechanisms such as interval-based restrictions help control the sensitivity of predictions to parameter perturbations.
The resulting separation is reminiscent of the function- and parameter-complexity penalties used in MDL and Bayesian model-selection approaches to SR~\cite{bartlett2023exhaustive,bartlett2023priors,desmond2026exhaustive}, but here it is obtained top-down from a generalization analysis rather than derived from coding-length or marginal-likelihood principles.

The remainder of this paper is organized as follows.
Section~\ref{sec:setup} introduces the SR setting and the expression-tree representation considered in this paper.
Section~\ref{sec:result} presents our main theorem, the generalization bound for GP-based SR.
Section~\ref{sec:proof} proves the main theorem using learning-theoretic analysis.
Section~\ref{sec:connect-empirics} discusses the connection between our theoretical result and empirical successes in GP-based SR.
Finally, Section~\ref{sec:conclusion} concludes the paper.

%% file: manuscripts/02_setup.tex
\section{Setup}
\label{sec:setup}

\subsection{Data, Risks, and Loss}
For SR, we consider a supervised regression setting.
Let $(x,y)\sim \cD$ with $x\in\cX\subseteq \R^d$ and $y\in\cY\subseteq \R$.
Given an i.i.d.\ sample $S=\{(x_i,y_i)\}_{i=1}^m$, define the population risk and the empirical risk:
\begin{align}
    L(f)=\E[\ell(f(x),y)] \enspace,
    \quad \what L_S(f)=\frac{1}{m}\sum_{i=1}^m \ell(f(x_i),y_i) \enspace ,
\end{align}
where $\ell$ is a loss function measuring prediction error.

\begin{definition}[Bounded Lipschitz loss]
We assume $\ell:\R\times\cY\to[0,1]$ is 1-Lipschitz in its first argument:
\begin{align}
    \abs{\ell(a,y)-\ell(a',y)}\le \abs{a-a'} \enspace,
    \quad \forall a,a'\in\R \enspace,\ \forall y\in\cY \enspace .
\end{align}
\end{definition}
This assumption is standard in statistical learning theory~\cite{mohri2018foundations} and allows generalization bounds to be expressed in terms of the complexity of the function class.
In practice, common losses such as mean squared error (MSE) can be accommodated through appropriate normalization or bounded-domain restrictions.

\subsection{Expression-tree Structures}
Our predictors are represented as expression trees, as commonly used in GP-based SR.
Internal nodes represent operators and leaves represent input variables or constants.

Let the operator vocabulary be partitioned by arity (number of arguments) up to $a_{\max}$:
\begin{align}
    \cO = \bigcup_{k=1}^{a_{\max}} \cO_k \enspace,
    \quad M_k=\abs{\cO_k} \enspace .
\end{align}
Typical examples include $\cO_1 = \{ \sin, \cos, \exp \}$ and $\cO_2 = \{ +, -, \times, \div \}$, where unary operators belong to $\cO_1$ and binary operators belong to $\cO_2$.

Leaves of the tree represent input variables and constants.
Let
\begin{align}
    \cL_0=\{x_1,\dots,x_d\} \cup \cC_{\mathrm{fix}} \enspace,
    \quad N_0=\abs{\cL_0} \enspace ,
\end{align}
where $\cC_{\mathrm{fix}}$ is an optional set of fixed constants (for example, $\{1, \pi \}$).
In addition, we allow a special leaf symbol $c$, representing a learnable constant.
During model fitting, this constant is optimized using the data.

\begin{definition}[Expression-tree structure]
An expression-tree structure $T$ is a rooted ordered tree with the following properties:
\begin{itemize}
    \item Each internal node has arity $k$ and is labeled by an operator in $\cO_k$.
    \item Each leaf is labeled by a symbol in $\cL_0\cup\{c\}$.
\end{itemize}
Let $p(T)$ be the number of leaves labeled by the learnable constant symbol $c$.
We fix a deterministic ordering of these leaves (e.g., depth-first order).
For a parameter vector $\theta\in\R^{p(T)}$, the function represented by the tree is $f_{T,\theta}:\cX\to\R$ obtained by assigning the value $\theta_j$ to the $j$-th constant leaf.
\end{definition}
Thus, each tree structure defines a family of functions parameterized by the constants $\theta$.

\subsection{Budgets and the Union Class}
To control model complexity, we restrict both the size and depth of expression trees.
Let $\mathrm{size}(T)$ be the number of nodes and $\mathrm{depth}(T)$ the maximum root-to-leaf path length measured in edges (with the root at depth $0$).

\begin{definition}[Budgets and induced classes]
Fix a size budget $s\in\N$ and a depth budget $D\in\N$ with $D \ge 1$.
We define the set of admissible tree structures
\begin{align}
    \cT_{s,D} = \{T:\ \mathrm{size}(T)\le s,\ \mathrm{depth}(T)\le D\} \enspace .
\end{align}
For a parameter radius $R>0$ and dimension $p\in\N$, define
$\Theta_p(R) = \{\theta\in\R^p: \norm{\theta}_2\le R\}$.
For each tree $T\in\cT_{s,D}$, we consider the class of functions
\begin{align}
    \cF_T(R) = \{f_{T,\theta}:\ \theta\in\Theta_{p(T)}(R)\} \enspace .
\end{align}
Finally, the overall function class allowed under the budgets is the union
\begin{align}
    \cF_{s,D}(R) = \bigcup_{T\in\cT_{s,D}} \cF_T(R) \enspace .
\end{align}
This class contains all expression-tree predictors whose structure satisfies the size and depth constraints and whose learnable constants lie within the radius-$R$ ball.
\end{definition}

\subsection{Assumptions}
To obtain generalization guarantees, we impose two standard conditions on the functions represented by expression trees.
Intuitively, these assumptions ensure that predictions remain bounded and that small changes in the learnable constants lead to controlled changes in the output.

We assume that predictions produced by budgeted trees are uniformly bounded.
\begin{definition}[Bounded predictions]
There exists $B>0$ such that for all $T\in\cT_{s,D}$, all $\theta\in\Theta_{p(T)}(R)$, and all $x\in\cX$,
\begin{align}
    \abs{f_{T,\theta}(x)}\le B \enspace .
\end{align}
\end{definition}
This condition prevents expressions from producing arbitrarily large values.
In practice, boundedness can be enforced by restricting the input domain, limiting parameter magnitudes, or using numerically stable operators.

We also assume that the output of the tree varies smoothly with respect to the learnable constants.
\begin{definition}[Coefficient-Lipschitz sensitivity]
There exists a quantity $G > 0$ such that for all $T \in \cT_{s,D}$, all $\theta,\theta'\in\Theta_{p(T)}(R)$, and all $x\in\cX$,
\begin{align}
    \abs{f_{T,\theta}(x)-f_{T,\theta'}(x)}\le G \,\norm{\theta-\theta'}_2 \enspace .
\end{align}
\end{definition}

The above condition is satisfied whenever the gradient with respect to the parameters is uniformly bounded: $\norm{\nabla_\theta f_{T,\theta}(x)}_2 \leq G$.
In SR practice, the constant $G$ typically depends on the stability of operators, restrictions on the input domain, and scaling of the expressions.

%% file: manuscripts/03_result.tex
\section{Main Result}
\label{sec:result}

The following theorem provides a unified generalization bound for SR models represented by expression trees.
The result explicitly decomposes the generalization error into two sources: structure selection and constant fitting.

\begin{theorem}[Unified generalization bound for GP-SR]
\label{thm:main}
Assume the bounded Lipschitz loss, bounded predictions, and coefficient-Lipschitz sensitivity assumptions.
Fix a size budget $s$, a depth budget $D$, and a parameter radius $R > 0$.
Then for any confidence level $\delta\in(0,1)$, with probability at least $1-\delta$ over the samples $S \sim \cD^m$, the following inequality holds simultaneously for all $f \in \cF_{s,D}(R)$:
\begin{align}
L(f)\
\le
\what L_S(f)
+ C_1\, R\, G \sqrt{\frac{s}{m}}
+ C_2\,B\sqrt{\frac{\log\abs{\cT_{s,D}}}{m}}
+ C_3\sqrt{\frac{\log(1/\delta)}{m}},
\label{eq:gen_bound}
\end{align}
where $C_1, C_2, C_3 > 0$ are absolute constants.

Moreover, in the unary/binary operator regime (arity at most 2), $\log|\cT_{s,D}|$ admits a depth-aware bound with exponential base
\begin{align}
\rho_{D} \;=\; 4\cos^2\!\left(\frac{\pi}{D+2} \right)\; < \; 4,
\end{align}
so that $\log|\cT_{s,D}|=O\!\left(\log s + s\log(\rho_{D} (M_1 + M_2) (N_0+1))\right)$ up to lower-order terms.
\end{theorem}

The first complexity term, $R\, G \sqrt{s/m}$, reflects the cost of fitting learnable constants within a fixed tree structure, while the second term, $B\sqrt{\log|\cT_{s,D}|/m}$, captures the combinatorial cost of selecting a tree structure.
The depth budget affects the latter term through $\rho_{D}$: smaller depth reduces the exponential growth rate of admissible structures.

%% file: manuscripts/04_proof.tex
\section{Proof}
\label{sec:proof}

This section proves the generalization bound stated in Theorem~\ref{thm:main}.
The argument proceeds in three main steps.
First, in Section~\ref{sec:finite_ub_rad}, we control the Rademacher complexity of a finite union of function classes, which corresponds to selecting a tree structure among the candidates in $\cT_{s,D}$.
Second, in Section~\ref{sec:constant-complexity}, we bound the Rademacher complexity of functions associated with a fixed tree structure, capturing the cost of fitting the learnable constants.
Third, in Section~\ref{sec:structure-complexity}, we bound the number of admissible tree structures, which determines the structure-selection term in the generalization bound.
Finally, combining these results yields the main theorem (Section~\ref{sec:combining_proof}).

\subsection{Finite Union Rademacher Complexity}
\label{sec:finite_ub_rad}

Our function class $\cF_{s,D}(R)$ is defined as a union over tree structures: $\cF_{s,D}(R) = \bigcup_{T\in\cT_{s,D}} \{f_{T,\theta}: \theta\in\Theta_{p(T)}(R)\}$.
Thus, bounding its Rademacher complexity requires controlling the complexity of a finite union of function classes.
This corresponds to the learning-theoretic cost of selecting a tree structure among the candidates in $\cT_{s,D}$.

We use empirical Rademacher complexity~\cite{bartlett2002rademacher,mohri2018foundations}:
\begin{align}
\Rad_S(\cG) = \E_{\sigma}\left[\sup_{g\in\cG}\frac{1}{m}\sum_{i=1}^m \sigma_i g(x_i)\right] \enspace ,
\end{align}
where $\sigma_i$ are i.i.d.\ Rademacher variables.
The following standard lemma bounds the Rademacher complexity of a finite union of classes.

\begin{lemma}[Finite union Rademacher bound]
\label{lem:finite-union}
Fix sample inputs $x_1,\dots,x_m$.
Let $\{\cF_j\}_{j=1}^M$ be function classes such that $\abs{f(x_i)}\le B$ for all $f\in\cup_j \cF_j$ and all $i$.
Then
\begin{align}
\Rad_S\!\left(\bigcup_{j=1}^M \cF_j\right)
\ \le\
\max_{1\le j\le M} \Rad_S(\cF_j)\ +\ B\sqrt{\frac{2\log M}{m}} \enspace .
\end{align}
\end{lemma}

\begin{proof}
Define
\begin{align}
Z_j(\sigma) := \sup_{f\in\cF_j}\frac{1}{m}\sum_{i=1}^m \sigma_i f(x_i) \enspace ,
\quad
Z(\sigma) := \max_{1\le j\le M} Z_j(\sigma) \enspace .
\end{align}
Then $\Rad_S(\cF_j)=\E_\sigma Z_j(\sigma)$ and $\Rad_S(\cup_j\cF_j)=\E_\sigma Z(\sigma)$.

For each $i$, modifying only $\sigma_i$ changes $Z_j(\sigma)$ by at most $2B/m$.
Hence the centered variables $Z_j(\sigma)-\E Z_j(\sigma)$ are sub-Gaussian with proxy variance $B^2/m$.
For sub-Gaussian variables with variance proxy $\sigma^2$, the standard maximal inequality~\cite{mohri2018foundations} holds that
\begin{align}
    \E\left[\max_{1\le j \le M} (Z_j-\E [Z_j]) \right] \le \sigma \sqrt{2 \log M} \enspace .
\end{align}
With $\sigma = B/\sqrt{m}$, it yields
\begin{align}
\E\left[\max_{1\le j\le M}(Z_j-\E [Z_j])\right]\ \le\ B\sqrt{\frac{2\log M}{m}} \enspace .
\end{align}
Therefore
\begin{align}
\E [Z] \le \max_j \E [Z_j] + B\sqrt{\frac{2\log M}{m}},
\end{align}
which is the claim.
\qed
\end{proof}

\begin{remark}[Structure selection as model choice]
Lemma~\ref{lem:finite-union} captures the learning-theoretic cost of selecting a tree structure.
Each class $\cF_T(R)$ corresponds to a fixed symbolic structure with different numerical constants.
The union over $T$ therefore represents the discrete model-selection problem inherent in SR.
The additional log term quantifies the combinatorial cost of searching over
alternative symbolic forms.
\end{remark}

\subsection{Constant-Fitting Complexity}
\label{sec:constant-complexity}

We now analyze the complexity of functions associated with a fixed tree structure $T$.
In this setting, the tree structure is fixed, and the only source of variability comes from the learnable constants appearing in the tree.

Let $T\in\cT_{s,D}$ be fixed and let $p=p(T)$ denote the number of constant leaves.
To control the complexity of this parameterized class, we measure distances between parameter vectors through the induced prediction pseudometric
\begin{align}
d_S(\theta,\theta') = \left(\frac{1}{m}\sum_{i=1}^m (f_{T,\theta}(x_i)-f_{T,\theta'}(x_i))^2\right)^{1/2} \enspace .
\end{align}

\begin{lemma}[Pseudometric domination]\label{lem:metric-dom}
Under the coefficient-Lipschitz sensitivity assumption, for all $\theta,\theta'\in\Theta_p(R)$,
\begin{align}
d_S(\theta,\theta')\le G\,\norm{\theta-\theta'}_2 \enspace .
\end{align}
\end{lemma}
\begin{proof}
For each $i$, $\abs{f_{T,\theta}(x_i)-f_{T,\theta'}(x_i)}\le G \norm{\theta-\theta'}_2$.
Squaring, averaging over $i$, and taking square roots gives the claim.
\qed
\end{proof}

\begin{remark}[Sensitivity to constants]
The quantity $G$ measures how strongly predictions change when the learnable constants are perturbed.
Large values of $G$ correspond to expressions whose outputs vary sharply with small parameter changes.
Such expressions are often numerically unstable and prone to overfitting,
a phenomenon frequently observed in GP-based SR when nested nonlinear operators are present.    
\end{remark}

We next bound the covering number of the parameter space under the pseudometric $d_S$.
Intuitively, this measures how many parameter vectors are needed to approximate all possible constant settings up to prediction error $\eps$ on the sample.

\begin{lemma}[Covering number under $d_S$]\label{lem:cover-ds}
Let $p=p(T)$. For any $\eps>0$,
\[
N(\eps,\Theta_p(R),d_S)\
\le\
\left(1 + \frac{2R\,G}{\eps}\right)^p \enspace .
\]
\end{lemma}
\begin{proof}
By Lemma~\ref{lem:metric-dom}, an $(\eps/G)$-cover in Euclidean norm is an $\eps$-cover under $d_S$.
By the Euclidean ball covering bound~\cite{vershynin2018high} (scaled to radius $R$),
$N(\eta,\Theta_p(R),\norm{\cdot}_2)\le (1+2R/\eta)^p$. Substitute $\eta=\eps/G$.
\qed
\end{proof}

To convert this covering-number bound into a Rademacher complexity estimate, we use Dudley’s entropy integral, which relates Rademacher complexity to
the logarithm of covering numbers.

\begin{theorem}[Dudley's entropy integral~\cite{dudley1967sizes}]
\label{thm:dudley}
There exists an absolute constant $C>0$ such that for any class $\cG$ of real-valued functions on $\{x_1,\dots,x_m\}$, the empirical Rademacher complexity satisfies
\begin{align}
\Rad_S(\cG)\
\le\
\frac{C}{\sqrt{m}}\int_0^\infty \sqrt{\log N(\eps,\cG,d_S)}\, d\eps \enspace ,
\end{align}
where $d_S$ is the empirical $L_2$ pseudo-metric.
\end{theorem}

Applying this result to the class $\cF_{T}(R)$ yields the following bound. 

\begin{theorem}[Fixed-structure constant-fitting bound]
\label{thm:fixed-structure}
For fixed $T$ with $p=p(T)$, there exists an absolute constant $C_{\mathrm{par}}>0$ such that
\begin{align}
\Rad_S(\cF_T(R))\ \le\ C_{\mathrm{par}}\,R\,G \sqrt{\frac{p}{m}} \enspace .
\end{align}
Consequently, since $p(T)\le \mathrm{size}(T)\le s$,
\begin{align}
\Rad_S(\cF_T(R))\ \le\ C_{\mathrm{par}}\,R\,G \sqrt{\frac{s}{m}}\enspace .
\end{align}
\end{theorem}

\begin{proof}
By Theorem~\ref{thm:dudley} and Lemma~\ref{lem:cover-ds},
\begin{align}
\Rad_S(\cF_T(R))
\le
\frac{C}{\sqrt m}\int_0^{\infty} \sqrt{p\log\left(1+\frac{2RG}{\eps}\right)}\,d\eps \enspace .
\end{align}
Truncate the integral at $2RG$ using the diameter bound from Lemma~\ref{lem:metric-dom}.
Then bound $\log(1+u)\le u$ and integrate as in standard Dudley computations, yielding
$\Rad_S(\cF_T(R))\le (4C)\,R G \sqrt{p/m}$.
Set $C_{\mathrm{par}}=4C$.
\qed
\end{proof}

\begin{remark}[Effect of learnable constants]
The bound grows with the square root of the number of constant leaves.
This reflects the intuition that introducing additional numerical parameters increases the flexibility of the expression even when its symbolic structure is fixed.
\end{remark}

\subsection{Structure-Selection Complexity}
\label{sec:structure-complexity}

We now bound the number of admissible tree structures in $\cT_{s,D}$.
This quantity determines the structure-selection term in the generalization bound through $\log|\cT_{s,D}|$.

The analysis proceeds in two steps.
First, we bound the number of possible tree shapes with bounded size and depth.
Second, we incorporate labelings by operators and terminals.

\begin{definition}[Unary/binary operator regime]
Assume all operators have arity at most 2.
Let $M_1=\abs{\cO_1}$ and $M_2=\abs{\cO_2}$ denote the number of unary and binary operators, respectively.
\end{definition}

We begin by bounding the number of rooted order tree shapes of bounded depth.
Classical results on planted plane trees (see, e.g., \cite{de1972average}) imply that the number of trees with bounded depth admits the following asymptotic growth:

\begin{theorem}[de Bruijn--Knuth--Rice asymptotic~\cite{de1972average}]\label{thm:bkr}
Let $A_{s,D}$ denote the number of planted plane trees with $s$ nodes and depth
\footnote{
While we use a depth measured in \emph{edges} $D$ (with the root at depth $0$), \cite{de1972average} uses a height in \emph{nodes} $h$ (counting the root as $1$).
The two are related by $D = h - 1$.
}
at most $D$.
Then for fixed $D$ and $s \to \infty$,
\begin{align}
A_{s,D} \sim
\frac{4^s}{D+2} \tan^2\!\left(\frac{\pi}{D+2}\right)\,
\cos^{2s}\!\left(\frac{\pi}{D+2}\right) \enspace .
\end{align}
\end{theorem}

The asymptotic expression reveals the exponential growth rate of $A_{s,D}$.
From this we obtain a simple exponential upper bound.

\begin{corollary}[Exponential upper bound with explicit base]\label{cor:bounded-height-base}
For each fixed $D \ge 1$ there exists $c_D>0$ such that for all $s \ge 0$,
\begin{align}
A_{s,D}\le c_D\,\rho_D^s \enspace ,
\quad
\rho_D := 4\cos^2\!\left(\frac{\pi}{D+2}\right) < 4 \enspace .
\end{align}
\end{corollary}
\begin{proof}
From Theorem~\ref{thm:bkr}, the dominant exponential factor is
$4^s\cos^{2s}(\pi/(D+2))=(4\cos^2(\pi/(D+2)))^s$.
The remaining factors depend only on $D$ and can therefore be absorbed into the constant $c_D$, yielding the stated bound.
\end{proof}



Combining the results yields the following estimate.


\begin{theorem}[Depth-aware structure counting]\label{thm:structure-count}
Assume the unary/binary operator regime.
There exists $c_D>0$ such that
\begin{align}
|\cT_{s,D}| &\le c_D\,(s+1)\,\Big(\rho_D (M_1+M_2) \,(N_0+1)\Big)^s \enspace ,
\end{align}
where 
\begin{align}
\rho_D &= 4\cos^2\!\left(\frac{\pi}{D+2}\right) < 4 \enspace .
\end{align}
Consequently,
\begin{align}
\log|\cT_{s,D}|
\ \le\
O(\log s)\ +\ s\log\rho_D \ +\ s\log (M_1 + M_2)\ +\ s\log(N_0+1) \enspace .
\end{align}
\end{theorem}

\begin{proof}
We derive an upper bound on $|\cT_{s,D}|$ by overcounting admissible tree structures.
For each $n \le s$, let $A_{n,D}$ denote the number of rooted ordered tree shapes with $n$ nodes and depth at most $D$.
By Corollary~\ref{cor:bounded-height-base}, we have
\begin{align}
    A_{n,D} \le c_D \rho_D^n \enspace .
\end{align}
For each such tree shape, each internal node can be labeled by one of at most $M_1 + M_2$ operators, and each leaf by one of $N_0 + 1$ terminal symbols.
Let $i$ and $\ell$ denote the number of internal nodes and leaves, respectively, so that $i + \ell = n$.
The number of labelings is at most
\begin{align}
    (M_1+M_2)^i (N_0+1)^{\ell} \le (M_1+M_2)^n (N_0+1)^n \enspace.
\end{align}
Summing over $n \le s$ gives
\begin{align}
|\cT_{s,D}| &\le \sum_{n=0}^s c_D \rho_D^n\, (M_1+M_2)^n\,(N_0+1)^{n} \\
&= c_D \sum_{n=0}^s \Big(\rho_D (M_1+M_2) (N_0+1)\Big)^n \\
&\le c_D (s+1)\Big(\rho_D (M_1+M_2) (N_0+1)\Big)^s \enspace .
\end{align}
We used the elementary bound for a geometric series
\begin{align}
    \sum_{n=0}^{s} a^n \le (s+1)a^s \quad (a \ge 1)
\end{align}
in the last step.
Taking logarithms gives the stated bound.
\qed
\end{proof}

\begin{remark}[Effect of depth constraints]
The depth parameter $D$ influences the growth rate of the hypothesis class through the exponential base
\begin{align}
    \rho_D = 4\cos^2\!\left(\frac{\pi}{D+2}\right) \enspace .
\end{align}
This quantity is strictly smaller than $4$ and increases monotonically toward $4$ as $D$ grows.
Consequently, smaller depth budgets reduce the exponential growth rate of admissible tree structures.
Intuitively, depth constraints remove highly elongated trees with long chains of nested operators.
Even when the size budget $s$ is fixed, such trees constitute a large fraction of the combinatorial search space.
Restricting depth therefore acts as a structural regularizer by limiting the variety of possible expression shapes.
\end{remark}

\subsection{Proof of the Main Theorem}
\label{sec:combining_proof}

We now combine the preceding ingredients.
A standard generalization bound for bounded 1-Lipschitz losses~\cite{mohri2018foundations} implies that, with probability at least $1-\delta$,
\begin{align}
\sup_{f\in\cF_{s,D}(R)}\big(L(f)-\what L_S(f)\big)
\ \le\ a\,\Rad_S(\cF_{s,D}(R))\ +\ b\sqrt{\frac{\log(1/\delta)}{m}} \enspace ,
\end{align}
for absolute constants $a, b > 0$.
Since the hypothesis class decomposes as $\cF_{s,D}(R)=\cup_{T\in\cT_{s,D}}\cF_T(R)$, Lemma~\ref{lem:finite-union} gives
\begin{align}
\Rad_S(\cF_{s,D}(R))
\ \le\
\max_{T\in\cT_{s,D}} \Rad_S(\cF_T(R))\ +\ B\sqrt{\frac{2\log|\cT_{s,D}|}{m}} \enspace.
\end{align}

By Theorem~\ref{thm:fixed-structure}, for every $T\in\cT_{s,D}$,
\begin{align}
\Rad_S(\cF_T(R))\le C_{\mathrm{par}}\,R\,G\sqrt{\frac{s}{m}} \enspace .
\end{align}
Hence the maximum over $T$ satisfies the same bound.

In the unary/binary regime, Theorem~\ref{thm:structure-count} provides an explicit bound on $\log|\cT_{s,D}|$, with exponential base
\begin{align}
\rho_D = 4\cos^2\!\left(\frac{\pi}{D+2}\right) \enspace .
\end{align}
Substituting these bounds and absorbing absolute constants into $C_1,C_2,C_3$ yields the statement of Theorem~\ref{thm:main}.
\qed

%% file: manuscripts/05_discussion.tex
\section{Connecting the Generalization Bound to GP Practice}
\label{sec:connect-empirics}

We now relate the generalization bound in \eqref{eq:gen_bound} to recurring design choices and empirical phenomena in GP-based SR.
The key organizing principle is the two-term decomposition of the bound: a \emph{structure-selection} term, governed by the size of the discrete hypothesis set, and a \emph{constant-fitting} term, governed by the sensitivity of predictions to learnable constants.
This perspective helps clarify how common GP heuristics act as forms of complexity control.

\subsection{Structure-Selection Term}
\label{subsec:connect-structure}

The structure-selection term reflects the combinatorial cost of choosing a tree structure from the hypothesis set $\cT_{s,D}$. Accordingly, mechanisms that restrict structural complexity can be interpreted as reducing $\log |\cT_{s,D}|$.

\paragraph{Bloat as hypothesis-class growth.}
A well-known empirical phenomenon in GP is \emph{bloat}~\cite{tackett1993genetic,langdon1997fitness,langdon1995evolving,soule2002analysis}, where program size increases even when improvements in training fitness have saturated.
In our framework, increasing the tree-size budget enlarges the hypothesis set $\cT_{s,D}$ and therefore increases the structure-selection term $\log |\cT_{s,D}|$ in the bound. From this perspective, bloat can be interpreted as expanding the effective hypothesis class considered by the algorithm, which may increase the risk of overfitting.

\paragraph{Parsimony pressure as capacity control.}
Parsimony pressure explicitly biases selection toward smaller trees, for example by adding a size penalty to fitness.
The accuracy--parsimony tradeoff and methods for parsimony pressure are classical topics in GP \cite{zhang1995balancing,poli2008parsimony}, including lexicographic variants \cite{luke2002lexicographic}.
Under our decomposition, parsimony pressure acts as an algorithmic proxy for reducing structural complexity by shifting the search toward smaller-$s$ structures.
Consequently, it can be interpreted as reducing $\log |\cT_{s,D}|$.

\paragraph{Depth limits.}
Depth constraints are ubiquitous in GP practice and are widely observed to stabilize evolution and improve test performance.
In our framework, depth constraints reduce the number of admissible structures and hence lower the growth rate of $|\cT_{s,D}|$.
They may also improve stability by limiting sensitivity amplification caused by deep compositions, although their most direct effect in the bound is through the structure-selection term.

\subsection{Constant-Fitting Term}
\label{subsec:connect-constants}

The constant-fitting term captures the complexity of optimizing continuous parameters within a fixed structure.
In the bound, this contribution appears through the factor $R G \sqrt{s/m}$, where $G$ measures the sensitivity of predictions to perturbations of learnable constants.

\paragraph{Constant optimization and sensitivity-driven overfitting.}
GP-based SR commonly introduces ephemeral random constants and may further optimize them via local search, nonlinear least squares~\cite{levenberg1944method,marquardt1963algorithm}, covariance matrix adaptation evolution strategy (CMA-ES)~\cite{hansen2001completely,hansen2016cma}, or related procedures~\cite{kommenda2020parameter}.
Constant optimization generally improves SR performance and can lead to smaller, better-generalizing trees.
However, when the fitted expression is sensitive to its constants, especially due to unstable operators or highly nonlinear compositions, aggressive fitting may still increase overfitting risk.
In the lens of \eqref{eq:gen_bound}, this behavior corresponds to a larger constant-fitting term when the induced function is highly sensitive to changes in constants, that is, when $G$ is large.
For example, expressions involving nested nonlinear operators such as $\exp(\exp(c_1 x))$ can exhibit high sensitivity to small changes in the learnable constants, leading to large values of $G$.
In contrast, simpler expressions such as $\sin(c_2 x) + c_3 x$ tend to be more stable.

\paragraph{Interval arithmetic and interval-aware operators.}
Interval arithmetic can detect expressions that produce undefined or unbounded outputs on the input domain, enabling rejection or penalization of such trees.
Keijzer showed that incorporating interval arithmetic can substantially improve the robustness of SR \cite{keijzer2003improving}.
Under our decomposition, these techniques primarily improve generalization by ruling out expressions that produce undefined or excessively large outputs, thereby helping to control $B$, while also discouraging unstable expressions and helping to reduce $G$.
Smooth operator definitions, such as analytic quotient operators, therefore fit naturally into our framework as mechanisms for limiting sensitivity amplification~\cite{ni2012use,nicolau2021choosing}.
Interval-based pruning may also remove invalid structures from the search space, which can indirectly reduce the effective structural complexity.

\paragraph{Linear scaling.}
Linear scaling, which fits an affine transformation of a candidate expression to the targets, is often effective in practice and is emphasized in interval-arithmetic-based SR \cite{keijzer2003improving,keijzer2004scaled}.
In our framework, linear scaling can be interpreted as concentrating continuous fitting into a small and relatively well-conditioned set of parameters.
This provides a structured way to reduce effective sensitivity without substantially enlarging the continuous search space.

\subsection{Summary and Implications}

Overall, the generalization bound highlights two complementary sources of complexity: the combinatorial growth of possible tree structures and the sensitivity of predictions to numerical constants.
This perspective suggests that effective SR systems should balance structural expressiveness and numerical stability.
Limiting tree size and depth reduces the combinatorial complexity of the hypothesis space, while stabilizing operators and controlling constant optimization reduce sensitivity to parameter perturbations.

%% file: manuscripts/06_conclusion.tex
\section{Conclusion}
\label{sec:conclusion}

In this paper, we presented a learning-theoretic analysis of SR models represented by expression trees.
We derived a generalization bound for GP-based SR that explicitly accounts for both structural complexity and parameter fitting.
Our result shows that the generalization gap can be decomposed into two interpretable components corresponding to the selection of the tree structure and the estimation of numerical constants.
This perspective provides a principled way to understand how common mechanisms in GP-based SR influence generalization.
In particular, our analysis offers a theoretical interpretation of widely used practices such as parsimony pressure, structural constraints, and the use of numerically stable operators.
By relating these mechanisms to the complexity terms appearing in the bound, our work helps bridge the gap between empirical heuristics and learning-theoretic principles.
We hope that this perspective will contribute to the development of more principled methods for controlling model complexity and improving generalization in SR.

Our analysis is based on worst-case complexity measures, including uniform sensitivity bounds and covering arguments over the full hypothesis class.
As a result, the obtained generalization bound is data-independent and may be conservative in practice.
Although the qualitative implications of our decomposition for model design remain largely unchanged, the resulting conservativeness may affect the quantitative tightness of the bound.
An important direction for future work is to develop data-dependent bounds that better capture the effective complexity of SR models on observed data, as has been extensively studied in other areas of machine learning~\cite{bousquet2002stability,dziugaite2017computing,arora2018stronger}.
For example, replacing worst-case sensitivity parameters by empirical measures of parameter sensitivity, or characterizing the effective set of structures explored by the search process, could lead to substantially tighter guarantees.
Such refinements may enable more nuanced trade-offs between expressiveness and stability, while preserving the core insights provided by our analysis.

Another promising direction is to leverage the proposed bound for the design of generalization-aware SR algorithms.
In many areas of machine learning, generalization bounds have been used to guide model design and training, for instance through PAC-Bayes objectives~\cite{dziugaite2017computing}, norm-based regularization in deep learning~\cite{neyshabur2015norm,bartlett2017spectrally}, and stability-driven analysis~\cite{bousquet2002stability}.
In SR, related ideas have been explored through structural risk minimization (SRM)-driven GP approaches~\cite{chen2018structural}, which incorporate complexity estimates, such as approximations of the VC dimension, into the evolutionary process.
However, these methods typically rely on empirical or surrogate complexity measures and are not derived from explicit generalization bounds for GP-style models.
In contrast, our analysis provides a unified learning-theoretic framework that explicitly decomposes generalization into structural and parametric components, which may serve as a principled basis for algorithm design.
This perspective opens the possibility of directly optimizing surrogate objectives derived from the generalization bound.

%% file: ref.bib
@book{mohri2018foundations,
  title={{Foundations of Machine Learning}},
  author={Mohri, Mehryar and Rostamizadeh, Afshin and Talwalkar, Ameet},
  year={2018},
  publisher={MIT press}
}

@book{vershynin2018high,
  title={{High-Dimensional Probability: An Introduction with Applications in Data Science}},
  author={Vershynin, Roman},
  volume={47},
  year={2018},
  publisher={Cambridge university press}
}

@incollection{de1972average,
  title={The average height of planted plane trees},
  author={de Bruijn, Nicolaas Govert and Knuth, Donald Ervin and Rice, SO},
  booktitle={Graph theory and computing},
  pages={15--22},
  year={1972},
  publisher={Elsevier}
}

@article{bartlett2002rademacher,
  title={{Rademacher and Gaussian Complexities: Risk Bounds and Structural Results}},
  author={Bartlett, Peter L and Mendelson, Shahar},
  journal={Journal of machine learning research},
  volume={3},
  number={Nov},
  pages={463--482},
  year={2002}
}

@article{dudley1967sizes,
  title={{The sizes of compact subsets of Hilbert space and continuity of Gaussian processes}},
  author={Dudley, Richard M},
  journal={Journal of Functional Analysis},
  volume={1},
  number={3},
  pages={290--330},
  year={1967},
  publisher={Elsevier}
}

@article{kommenda2020parameter,
  title={{Parameter identification for symbolic regression using nonlinear least squares}},
  author={Kommenda, Michael and Burlacu, Bogdan and Kronberger, Gabriel and Affenzeller, Michael},
  journal={Genetic Programming and Evolvable Machines},
  volume={21},
  number={3},
  pages={471--501},
  year={2020},
  publisher={Springer}
}

@article{zhang1995balancing,
  title={{Balancing Accuracy and Parsimony in Genetic Programming}},
  author={Zhang, Byoung-Tak and M{\"u}hlenbein, Heinz},
  journal={Evolutionary Computation},
  volume={3},
  number={1},
  pages={17--38},
  year={1995},
  publisher={MIT Press One Rogers Street, Cambridge, MA 02142-1209, USA journals-info~…}
}

@inproceedings{poli2008parsimony,
  title={{Parsimony Pressure Made Easy}},
  author={Poli, Riccardo and McPhee, Nicholas Freitag},
  booktitle={Proceedings of the 10th annual conference on Genetic and evolutionary computation},
  pages={1267--1274},
  year={2008}
}

@inproceedings{luke2002lexicographic,
  title={{Lexicographic Parsimony Pressure}},
  author={Luke, Sean and Panait, Liviu},
  booktitle={Proceedings of the 4th Annual Conference on Genetic and Evolutionary Computation},
  pages={829--836},
  year={2002}
}

@inproceedings{keijzer2003improving,
  title={{Improving Symbolic Regression with Interval Arithmetic and Linear Scaling}},
  author={Keijzer, Maarten},
  booktitle={European Conference on Genetic Programming},
  pages={70--82},
  year={2003},
  organization={Springer}
}

@article{chen2018structural,
  title={{Structural Risk Minimization-Driven Genetic Programming for Enhancing Generalization in Symbolic Regression}},
  author={Chen, Qi and Zhang, Mengjie and Xue, Bing},
  journal={IEEE Transactions on Evolutionary Computation},
  volume={23},
  number={4},
  pages={703--717},
  year={2018},
  publisher={IEEE}
}

@book{koza1992genetic,
  title={{Genetic Programming: On the Programming of Computers by Means of Natural Selection}},
  author={Koza, John R},
  volume={1},
  year={1992},
  publisher={MIT Press}
}

@book{poli2008field,
  title={{A Field Guide to Genetic Programming}},
  author={Poli, Riccardo and Langdon, William B and McPhee, Nicholas F and Koza, John R},
  year={2008},
  publisher={Lulu. com}
}

@article{levenberg1944method,
  title={{A method for the solution of certain non-linear problems in least squares}},
  author={Levenberg, Kenneth},
  journal={Quarterly of applied mathematics},
  volume={2},
  number={2},
  pages={164--168},
  year={1944}
}

@article{marquardt1963algorithm,
  title={{An Algorithm for Least-Squares Estimation of Nonlinear Parameters}},
  author={Marquardt, Donald W},
  journal={Journal of the society for Industrial and Applied Mathematics},
  volume={11},
  number={2},
  pages={431--441},
  year={1963},
  publisher={SIAM}
}

@article{hansen2016cma,
  title={{The CMA Evolution Strategy: A Tutorial}},
  author={Hansen, Nikolaus},
  journal={arXiv preprint arXiv:1604.00772},
  year={2016}
}

@article{hansen2001completely,
  title={{Completely Derandomized Self-Adaptation in Evolution Strategies}},
  author={Hansen, Nikolaus and Ostermeier, Andreas},
  journal={Evolutionary computation},
  volume={9},
  number={2},
  pages={159--195},
  year={2001},
}

@article{keijzer2004scaled,
  title={{Scaled Symbolic Regression}},
  author={Keijzer, Maarten},
  journal={Genetic Programming and Evolvable Machines},
  volume={5},
  number={3},
  pages={259--269},
  year={2004},
  publisher={Springer}
}

@incollection{langdon1997fitness,
  title={{Fitness Causes Bloat}},
  author={Langdon, William B and Poli, Riccardo},
  booktitle={Soft Computing in Engineering Design and Manufacturing},
  pages={13--22},
  year={1997},
  publisher={Springer}
}

@inproceedings{tackett1993genetic,
  title={{Genetic Programming for Feature Discovery and Image Discrimination}},
  author={Tackett, Walter Alden},
  booktitle={Proceedings of the 5th International Conference on Genetic Algorithms},
  pages={303--311},
  year={1993}
}

@inproceedings{langdon1995evolving,
  title={{Evolving Data Structures with Genetic Programming}},
  author={Langdon, William B},
  booktitle={Proceedings of the 6th International Conference on Genetic Algorithms},
  pages={295--302},
  year={1995}
}

@article{soule2002analysis,
  title={{An Analysis of the Causes of Code Growth in Genetic Programming}},
  author={Soule, Terence and Heckendorn, Robert B},
  journal={Genetic Programming and Evolvable Machines},
  volume={3},
  number={3},
  pages={283--309},
  year={2002},
  publisher={Springer}
}

@book{kronberger2024symbolic,
  title={{Symbolic Regression}},
  author={Kronberger, Gabriel and Burlacu, Bogdan and Kommenda, Michael and Winkler, Stephan M and Affenzeller, Michael},
  year={2024},
  publisher={Chapman and Hall/CRC}
}

@article{udrescu2020ai,
  title={{AI Feynman: A Physics-Inspired Method for Symbolic Regression}},
  author={Udrescu, Silviu-Marian and Tegmark, Max},
  journal={Science advances},
  volume={6},
  number={16},
  pages={eaay2631},
  year={2020},
  publisher={American Association for the Advancement of Science}
}

@article{brunton2016discovering,
  title={{Discovering governing equations from data by sparse identification of nonlinear dynamical systems}},
  author={Brunton, Steven L and Proctor, Joshua L and Kutz, J Nathan},
  journal={Proceedings of the national academy of sciences},
  volume={113},
  number={15},
  pages={3932--3937},
  year={2016},
  publisher={National Academy of Sciences}
}

@article{wang2019symbolic,
  title={{Symbolic Regression in Materials Science}},
  author={Wang, Yiqun and Wagner, Nicholas and Rondinelli, James M},
  journal={MRS communications},
  volume={9},
  number={3},
  pages={793--805},
  year={2019},
  publisher={Cambridge University Press}
}

@article{dziugaite2017computing,
  title={{Computing Nonvacuous Generalization Bounds for Deep (Stochastic) Neural Networks with Many More Parameters than Training Data}},
  author={Dziugaite, Gintare Karolina and Roy, Daniel M},
  journal={arXiv preprint arXiv:1703.11008},
  year={2017}
}

@article{bousquet2002stability,
  title={{Stability and Generalization}},
  author={Bousquet, Olivier and Elisseeff, Andr{\'e}},
  journal={Journal of Machine Learning Research},
  volume={2},
  number={Mar},
  pages={499--526},
  year={2002}
}

@article{bartlett2017spectrally,
  title={{Spectrally-normalized margin bounds for neural networks}},
  author={Bartlett, Peter L and Foster, Dylan J and Telgarsky, Matus J},
  journal={Advances in neural information processing systems},
  volume={30},
  year={2017}
}

@inproceedings{neyshabur2015norm,
  title={{Norm-Based Capacity Control in Neural Networks}},
  author={Neyshabur, Behnam and Tomioka, Ryota and Srebro, Nathan},
  booktitle={Conference on Learning Theory},
  pages={1376--1401},
  year={2015},
  organization={PMLR}
}

@inproceedings{arora2018stronger,
  title={{Stronger generalization bounds for deep nets via a compression approach}},
  author={Arora, Sanjeev and Ge, Rong and Neyshabur, Behnam and Zhang, Yi},
  booktitle={International Conference on Machine Learning},
  pages={254--263},
  year={2018},
  organization={PMLR}
}

@article{bartlett2023exhaustive,
  title={{Exhaustive Symbolic Regression}},
  author={Bartlett, Deaglan J and Desmond, Harry and Ferreira, Pedro G},
  journal={IEEE Transactions on Evolutionary Computation},
  volume={28},
  number={4},
  pages={950--964},
  year={2023},
  publisher={IEEE}
}

@inproceedings{bartlett2023priors,
  title={{Priors for Symbolic Regression}},
  author={Bartlett, Deaglan and Desmond, Harry and Ferreira, Pedro},
  booktitle={Proceedings of the Companion Conference on Genetic and Evolutionary Computation},
  pages={2402--2411},
  year={2023}
}

@article{desmond2026exhaustive,
  title={{(Exhaustive) Symbolic Regression and model selection by minimum description length}},
  author={Desmond, Harry},
  journal={Philosophical Transactions of the Royal Society A: Mathematical, Physical and Engineering Sciences},
  volume={384},
  number={2317},
  pages={20240584},
  year={2026},
  publisher={The Royal Society}
}

@article{ni2012use,
  title={{The Use of an Analytic Quotient Operator in Genetic Programming}},
  author={Ni, Ji and Drieberg, Russ H and Rockett, Peter I},
  journal={IEEE Transactions on Evolutionary Computation},
  volume={17},
  number={1},
  pages={146--152},
  year={2012},
  publisher={IEEE}
}

@article{nicolau2021choosing,
  title={Choosing function sets with better generalisation performance for symbolic regression models},
  author={Nicolau, Miguel and Agapitos, Alexandros},
  journal={Genetic programming and evolvable machines},
  volume={22},
  number={1},
  pages={73--100},
  year={2021},
  publisher={Springer}
}
